\documentclass[11pt]{article}

\usepackage[numbers]{natbib}
\usepackage[margin=1in]{geometry}

\usepackage[utf8]{inputenc} 
\usepackage[T1]{fontenc}    
\usepackage{hyperref}       
\usepackage{url}            
\usepackage{booktabs}       
\usepackage{amsfonts}       
\usepackage{nicefrac}       
\usepackage{microtype}      
\usepackage{xcolor}         

\usepackage{amsmath,amssymb,amsthm}
\usepackage{graphicx}
\usepackage{enumitem}
\usepackage{algorithm}
\usepackage{algorithmic}
\usepackage{multirow}
\usepackage{pgfplots}
\pgfplotsset{compat=1.18}
\usepackage{authblk}

\newtheorem{definition}{Definition}
\newtheorem{assumption}{Assumption}
\newtheorem{proposition}{Proposition}
\newtheorem{theorem}{Theorem}

\newcommand{\R}{\mathbb{R}}
\newcommand{\norm}[1]{\left\lVert #1 \right\rVert}
\newcommand{\ip}[2]{\left\langle #1,#2 \right\rangle}

\title{Hallucination-Aware Diffusion Sampling for \\Inverse Problems via Robust Prior Updates}

\author[1]{Pengfei Jin\textsuperscript{*}}
\author[1,2]{Yiqi Tian \textsuperscript{*}}
\author[1]{Kailong Fan}
\author[1]{Bingjie Qi}
\author[1]{Quanzheng Li\textsuperscript{\dag}}

\affil[1]{Center for Advanced Medical Computing and Analysis, Massachusetts General Hospital and Harvard Medical School, Boston, MA 02114}
\affil[2]{Department of Industrial Engineering, University of Pittsburgh, Pittsburgh, PA 15261}
\date{}

\begin{document}
\maketitle

\renewcommand{\thefootnote}{\fnsymbol{footnote}}
\footnotetext[1]{Equal contribution.}
\footnotetext[2]{Corresponding author. Email: \texttt{li.quanzheng@mgh.harvard.edu}.}
\renewcommand{\thefootnote}{\arabic{footnote}}

\begin{abstract}
Diffusion-based inverse problem solvers can produce realistic reconstructions, but realism alone does not ensure that the recovered details are supported by the measurement. We study this failure as measurement-conditioned hallucination: visually meaningful content that is either implausible or inconsistent with the measured instance. Our analysis separates Bayes-rule-based diffusion inverse solvers into a prior update and a measurement-conditioning step, showing that hallucinated content can enter through the prior-side proposal before the measurement correction is applied. Motivated by this view, we propose Robust Prior Update (RPU), a solver-level module that probes the local stability of the diffusion prior update, re-anchors the resulting displacement at the current iterate, and leaves the measurement update unchanged. We instantiate RPU in DPS and evaluate it on FFHQ and ImageNet inverse problems using automatic metrics and human faithfulness studies. On FFHQ, RPU improves PSNR and LPIPS over DPS across box inpainting, Gaussian deblurring, and motion deblurring. In human judgments, RPU receives 91.9\% of blind non-tie majority preferences and 91.1\% of ground-truth-assisted non-tie preferences on FFHQ box inpainting, while the ImageNet Gaussian reader study is tie-heavy but favors RPU among non-tie cases. These results support a targeted claim: robustifying the prior update can improve instance faithfulness in diffusion inverse solvers, especially when the prior shapes weakly constrained content.
\end{abstract}

\section{Introduction}

Inverse problems seek to recover an unknown signal from indirect, incomplete, or corrupted measurements. They arise across a wide range of imaging and reconstruction tasks, including accelerated MRI, sparse-view CT, image inpainting, super-resolution, deblurring, compressed sensing, and phase retrieval \citep{lustig2008compressed,dong2015image, saharia2022image,candes2006robust,donoho2006compressed,fienup1982phase}. A common feature of these problems is that the forward measurement process discards information: a masking operator removes pixels or Fourier coefficients, a blur operator suppresses high-frequency details, and a phase-retrieval operator loses phase information. As a result, the inverse map is generally ill-posed, and a single observation may be compatible with many plausible underlying signals \citep{tarantola2005inverse,scarlett2023theoretical}. Solving such problems therefore requires not only enforcing consistency with the measurement, but also imposing a prior that selects among the many feasible reconstructions.

Diffusion models provide a powerful prior for resolving the ambiguity of inverse problems, but the same generative strength can also introduce content that is not determined by the measurement. This risk is related to hallucination in generative models, where outputs may be unfounded, unfaithful, or not grounded in the underlying data \citep{ji2023survey}. In unconditional diffusion generation, such failures can appear as visually implausible or out-of-support samples, including distorted anatomy, abnormal hands, or incoherent object structures \citep{aithal2024understanding,tian2025rods}. In inverse problems, however, the more concerning failure mode can be subtler: a reconstruction may look realistic while containing details that are not supported by the observed measurement. This measurement-conditioned hallucination is difficult to detect by visual realism alone, because perceptual plausibility does not guarantee measurement faithfulness.

This distinction makes measurement-conditioned hallucination especially important in practice. Many real-world uses of generative reconstruction are conditional rather than purely unconditional: in seismic imaging, audio restoration, and medical imaging, the goal is to recover a signal faithful to a specific observation, not merely to generate a realistic sample \citep{lailly1983sequence,moliner2023solving,song2021solving,chung2022score}. In these settings, visual plausibility can be misleading, since unsupported changes may affect interpretation, diagnosis, or downstream analysis. Prior work has recognized hallucination in inverse problems and conditional reconstruction, including false structures in tomographic reconstruction, instability in AI-based inverse solvers, hallucination metrics for generative reconstruction, and hallucination reduction for conditional medical reconstruction \citep{bhadra2021hallucinations,gottschling2025troublesome,tivnan2024hallucination,kim2025tackling}. However, relatively few works study how such hallucination errors arise inside diffusion-based inverse solvers. This leaves open a mechanistic question: when do unsupported details enter the diffusion-based inverse solving process, and how do they persist despite measurement-consistency corrections?

We address this question by studying the update process of Bayes-rule-based diffusion inverse solvers. In these methods, the reverse process couples two roles: the diffusion prior fills in missing or ambiguous content, while the measurement term constrains the reconstruction to the observation. Therefore, the hallucination can appear either as visually implausible artifacts or as realistic-looking details unsupported by the measurement.

\begin{figure}[t]
    \centering
    \includegraphics[width=0.85\textwidth]{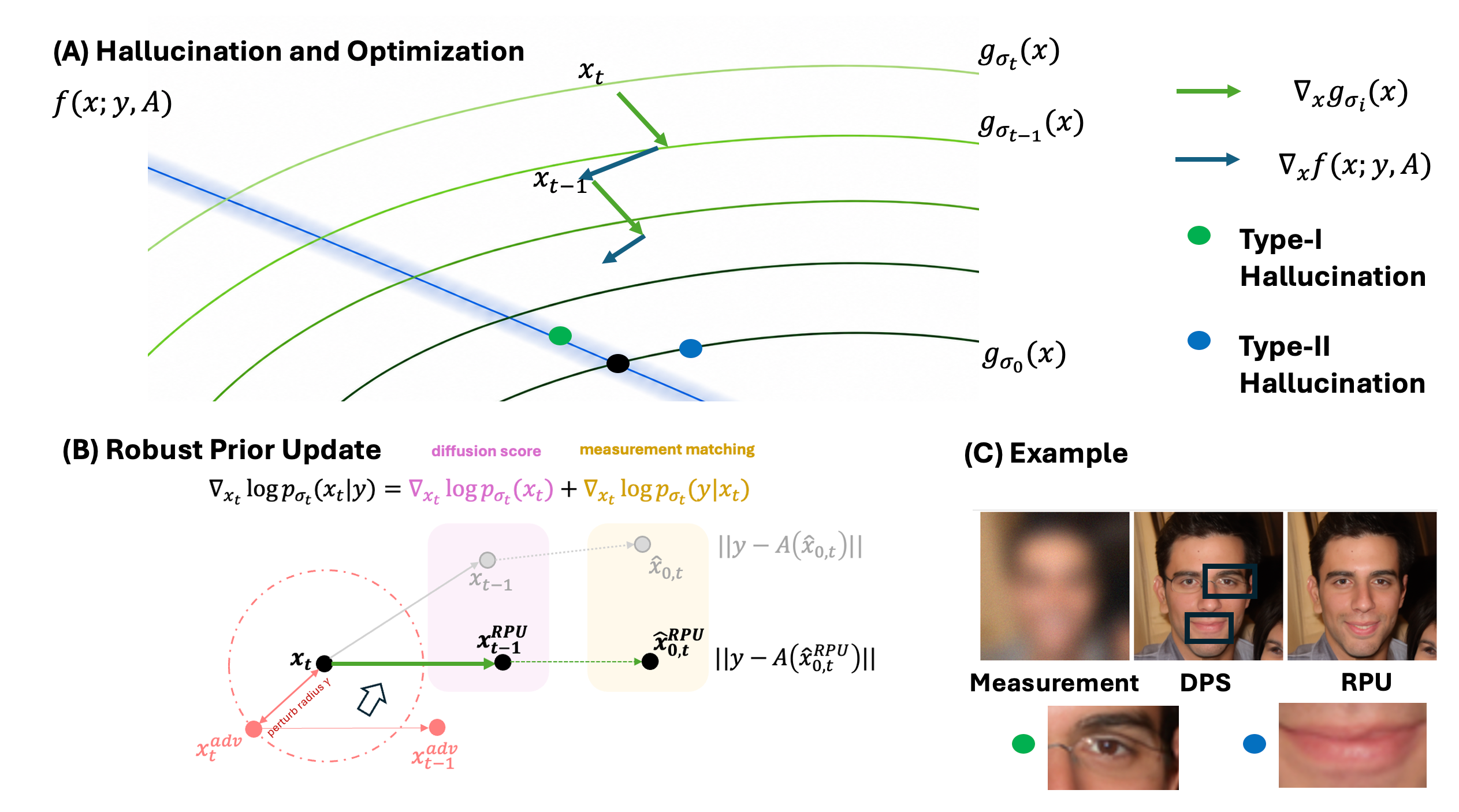}
    \caption{Overview. (A) Diffusion inverse solvers balance measurement fidelity \(f(x;y,A)\) and the diffusion prior \(g_{\sigma_i}(x)\); errors in these directions can yield visually implausible Type-I or measurement-unsupported Type-II hallucinations. (B) RPU probes the prior update, re-anchors the resulting displacement at the current iterate, and keeps the measurement condition unchanged. (C) In Gaussian deblurring, DPS hallucinates a partial eyeglass and unsupported mouth detail compared with RPU..}
    \label{fig:intro-roadmap}
\end{figure}

Our contributions are summarized as follows:
\begin{enumerate}[leftmargin=1.5em]
    \item We define hallucination in generative inverse problems through two cases: Type-I hallucination, which captures visually implausible artifacts, and Type-II hallucination, which captures plausible but measurement-unsupported details.

    \item We derive an alternating-optimization view of Bayes-rule-based diffusion inverse solvers, separating the measurement-consistency step from the prior-update step and identifying the prior update as a natural intervention point.

    \item We propose \textbf{RPU}, a robust prior-update module that stabilizes the prior-induced displacement while keeping the measurement update fixed.

    \item We instantiate RPU in DPS and evaluate it with quantitative metrics and human faithfulness judgments on FFHQ and ImageNet inverse problems.
\end{enumerate}

\section{Related Work}

\paragraph{Diffusion-based inverse problem solvers.}

Diffusion inverse-problem solvers are commonly organized by how they combine a pretrained diffusion prior with the measurement model~\citep{daras2024survey}. Measurement-guided methods approximate the likelihood or posterior score, including MCG~\citep{chung2022mcg}, DPS~\citep{chung2023dps}, DDRM~\citep{kawar2022ddrm}, DDNM~\citep{wang2022ddnm}, and \(\Pi\)GDM~\citep{song2023pigdm}. Other lines use variational or plug-and-play objectives, such as DiffPIR~\citep{zhu2023diffpir} and RED-Diff~\citep{mardani2023variational}; optimize latent variables or initial noise as in CSGM~\citep{bora2017compressed} and Ambient Diffusion~\citep{daras2023solving}; or design more exact Monte Carlo posterior samplers such as TDS~\citep{wu2023practical}. We focus on Bayes-rule-based solvers, especially DPS~\citep{chung2023dps}, and ask how hallucination can be shaped by the prior-side proposal rather than only by the measurement-matching approximation.

\paragraph{Hallucination in generative and reconstruction models.}
In unconditional generation, hallucination often means out-of-support or artifact-like samples, with diffusion failures linked to mode interpolation~\citep{aithal2024understanding} and unstable reverse updates~\citep{tian2025rods}. In conditional reconstruction, the failure is measurement-dependent: a plausible image may still contain structures unsupported by the observation. Prior work studies this issue through tomographic null-space analysis~\citep{bhadra2020hallucinations}, information-theoretic limits of generative restoration~\citep{cohen2024lookstoogood}, hallucination indexes~\citep{tivnan2024hallucinationindex}, and adaptive mitigation in DynamicDPS~\citep{kim2025dynamicdps}. These works motivate hallucination as a reconstruction reliability problem; our focus is the solver process by which such content enters a Bayes-rule-based diffusion trajectory.

\paragraph{Robust optimization and perturbation-based robustness.}
Robust optimization accounts for uncertainty by optimizing against worst-case perturbations in a prescribed set~\citep{ben2002robust,bertsimas2011theory,lu2024two}. Related machine-learning tools include adversarial training and perturbation-based regularization for local stability~\citep{madry2017towards,miyato2018virtual,foret2020sharpness}. We use this perspective at the sampler level, robustifying the prior-side update that injects generative content into the inverse-solving trajectory.

\section{Preliminary and Hallucination Definition}
\label{sec:problem}

\subsection{Diffusion-based inverse problem solver}

We consider an inverse problem $y = A(x^\star) + n$, where $x^\star \in \R^d$ is the unknown clean image, $A$ is the forward measurement operator, $n$ is observation noise, and $y$ is the observed measurement. A standard reconstruction objective balances measurement consistency $f$ with an image prior $g$:
\begin{equation}
    \min_x
    \underbrace{f(x;y,A)}_{\text{measurement fidelity}}
    +
    \underbrace{g(x)}_{\text{image prior / regularization}} .
    \label{eq:inverse-objective}
\end{equation}
Specifically, for Gaussian measurement noise, we use $f(x;y,A)=\frac{1}{2\sigma_y^2}\norm{A(x)-y}^2$.

Diffusion-based inverse-problem solvers use the same two ingredients, but usually through posterior sampling rather than direct minimization of \eqref{eq:inverse-objective}. Many methods use a pretrained diffusion model as the image prior and guide the reverse trajectory toward the measurement-conditioned distribution \citep{song2021scoremedical,chung2023dps,daras2024survey}. At noise level $\sigma$, let $p_\sigma$ denote the smoothed image distribution and $s_\sigma(x)=\nabla_x\log p_\sigma(x)$ its score. Equivalently, with $\phi_\sigma(x)=-\log p_\sigma(x)$, the diffusion prior can be written as $g_\sigma(x)=\lambda\phi_\sigma(x)$, with $\nabla\phi_\sigma(x)=-s_\sigma(x)$. A common class of methods can be viewed as Bayes-rule-based posterior solvers, where Bayes' rule decomposes the posterior score as
\begin{equation}
    \nabla_{x_t}\log p_{\sigma_t}(x_t\mid y)
    =
    \underbrace{\nabla_{x_t}\log p_{\sigma_t}(x_t)}_{\text{diffusion score}}
    +
    \underbrace{\nabla_{x_t}\log p_{\sigma_t}(y\mid x_t)}_{\text{measurement matching}} .
    \label{eq:bayes-score-decomposition}
\end{equation}
The first term is approximated by $s_{\sigma_t}(x_t)$. The second term is the measurement-matching direction. For example, DPS \citep{chung2023dps} approximates this term through a denoised estimate $\hat{x}_0(x_t)$:
\begin{equation}
    \nabla_{x_t}\log p_{\sigma_t}(y\mid x_t)
    \approx
    -\rho_t
    \nabla_{x_t}
    \ell_y(\hat{x}_0(x_t)),
    \qquad
    \ell_y(\hat{x}_0)
    =
    \frac{1}{2\sigma_y^2}\norm{A(\hat{x}_0)-y}^2 .
    \label{eq:dps-measurement-gradient}
\end{equation}

\subsection{Hallucination Definition}
\label{subsec:hallucination-definition}

\begin{definition}[Inverse-problem hallucination]
\label{def:ip-hallucination}
Given a measurement $y=A(x^\star)+n$ and a reconstruction $\hat{x}$ produced from $(y,A)$, we refer to image content in $\hat{x}$ as hallucinated when it is introduced or substantially amplified by the reconstruction procedure, appears as meaningful image content rather than random noise, and is not faithful to the corresponding ground-truth instance $x^\star$ in a way justified by the measurement model.
\end{definition}

This definition separates hallucination from general reconstruction error. Blurring, missing texture, or small pixel-level deviations may reduce reconstruction quality, but they are not necessarily hallucinations. The failure of interest is solver-induced content that appears as meaningful image structure but is not supported by the measured instance. We use Type-I hallucination to describe visible solver-induced artifacts or unstable structures that can often be identified from $\hat{x}$ alone, and Type-II hallucination to describe plausible-looking details that appear reasonable in isolation but are inconsistent with $x^\star$; in our evaluation, the latter is assessed by comparing reconstructions with the ground truth.

Standard metrics provide useful context but do not fully capture instance-level faithfulness. PSNR, LPIPS, FID, and measurement residuals measure reconstruction quality, perceptual similarity, distribution-level realism, or consistency with $y$, but they do not directly assess whether visible details agree with $x^\star$. We therefore report these metrics as context and use human comparison studies to assess perceived faithfulness to the ground truth.

\section{Alternating Optimization View}
\label{sec:alternating}

Bayes-rule-based diffusion inverse solvers, including DPS-style posterior guidance, combine a diffusion-prior direction with a measurement-matching direction. This mirrors the two terms in \eqref{eq:inverse-objective}: $f$ enforces consistency with $(y,A)$, while $g_\sigma=\lambda\phi_\sigma$ is the smoothed diffusion prior. 

\subsection{Continuation alternating subproblems}

For a fixed smoothing level $\sigma$, define
\begin{equation}
    F_\sigma(x) := f(x;y,A) + g_\sigma(x),
    \qquad
    g_\sigma(x)=\lambda\phi_\sigma(x),
    \qquad
    \nabla \phi_\sigma(x)=-s_\sigma(x).
    \label{eq:smoothed-objective}
\end{equation}
We idealize the two directions as local proximal subproblems:
\begin{align}
    z_k
    &=
    \arg\min_z
    \left\{
    f(z;y,A)
    +
    \frac{1}{2\eta_f}\norm{z-x_k}^2
    \right\}, \label{eq:measurement-subproblem}\\
    x_{k+1}
    &=
    \arg\min_x
    \left\{
    g_{\sigma_k}(x)
    +
    \frac{1}{2\eta_p}\norm{x-z_k}^2
    \right\}.
    \label{eq:prior-subproblem}
\end{align}

The proximal terms are important conceptually: they make each subproblem a local correction around the current iterate, rather than an attempt to minimize the measurement or prior term in isolation. This matches practical diffusion inverse solvers, where both the data-consistency step and the score-model step are small updates inside a decreasing-noise trajectory.

We do not require closed-form solutions. Let $\mathcal{M}_{\eta_f}$ and $\mathcal{P}_{\eta_p,\sigma_k}$ denote exact or approximate oracles for \eqref{eq:measurement-subproblem} and \eqref{eq:prior-subproblem}:
\begin{equation}
    z_k = \mathcal{M}_{\eta_f}(x_k),
    \qquad
    x_{k+1}=\mathcal{P}_{\eta_p,\sigma_k}(z_k).
    \label{eq:oracle-calls}
\end{equation}
Different oracle approximations give different concrete solvers. With a one-step first-order approximation,
\begin{equation}
    z_k = x_k-\alpha_k
\nabla f(x_k),
    \qquad
    x_{k+1}
    =
    z_k-\beta_k
\nabla g_{\sigma_k}(z_k)
    =
    z_k+\beta_k\lambda s_{\sigma_k}(z_k),
    \label{eq:first-order-split}
\end{equation}
which is the split update analyzed below.

\subsection{Local fixed-smoothing analysis}

The following local, fixed-$\sigma$ statement gives a descent interpretation for \eqref{eq:first-order-split}; it is not a global convergence theorem for diffusion sampling.

\begin{assumption}[Local fixed-$\sigma$ regularity]
\label{assump:fixedsigma}
On a local basin $\mathcal{B}_\sigma$, assume: (i) $f$ has $L_f$-Lipschitz gradient; (ii) $g_\sigma$ has $L_{g,\sigma}$-Lipschitz gradient; (iii) the iterates generated by \eqref{eq:first-order-split} remain in $\mathcal{B}_\sigma$; and (iv) there exist constants $\kappa_f(\sigma),\kappa_g(\sigma)\ge 0$ such that, with $z=x-\alpha\nabla f(x)$,
\begin{equation}
    \ip{\nabla g_\sigma(x)}{\nabla f(x)} \ge -\kappa_f(\sigma)\norm{\nabla f(x)}^2,
    \qquad
    \ip{\nabla f(z)}{\nabla g_\sigma(z)} \ge -\kappa_g(\sigma)\norm{\nabla g_\sigma(z)}^2.
\end{equation}
\end{assumption}

The inner-product conditions allow the two directions to be imperfectly aligned, but rule out destructive interference that would overwhelm descent from either step. They are local conditions because the score prior is only expected to be smooth and useful within the basin currently tracked by the sampler.

Define
\begin{equation}
    a_\sigma := \alpha(1-\kappa_f(\sigma)) - \tfrac{\alpha^2}{2}(L_f+L_{g,\sigma}),
    \qquad
    b_\sigma := \beta(1-\kappa_g(\sigma)) - \tfrac{\beta^2}{2}(L_f+L_{g,\sigma}).
\end{equation}

\begin{proposition}[Fixed-$\sigma$ split-step descent]
\label{prop:fixedsigma-descent}
Under Assumption~\ref{assump:fixedsigma}, if $a_\sigma>0$ and $b_\sigma>0$, then the update $z=x-\alpha\nabla f(x)$, $x^+=z-\beta\nabla g_\sigma(z)$ satisfies
\begin{equation}
    F_\sigma(x^+) \le F_\sigma(x) - a_\sigma \norm{\nabla f(x)}^2 - b_\sigma \norm{\nabla g_\sigma(z)}^2 .
    \label{eq:local-split-descent}
\end{equation}
If, in addition, $F_\sigma$ satisfies a local Polyak--Lojasiewicz inequality on $\mathcal{B}_\sigma$ with constant $\mu_\sigma>0$, then with $c_\sigma := \min\{a_\sigma/[2(1+\alpha L_{g,\sigma})^2],\, b_\sigma/2\} > 0$, whenever $0<2\mu_\sigma c_\sigma<1$,
\begin{equation}
    F_\sigma(x^+) - F_\sigma^\star \le (1-2\mu_\sigma c_\sigma)\left(F_\sigma(x)-F_\sigma^\star\right).
    \label{eq:fixed-sigma-contraction}
\end{equation}
\end{proposition}
The decreasing-$\sigma$ schedule then acts as a continuation device: large $\sigma$ smooths the prior landscape, and small $\sigma$ sharpens it toward the target image prior. Chaining the fixed-$\sigma$ statement therefore requires basin tracking, because descent at one smoothing level is useful only if the next level starts in a compatible local region.

\begin{theorem}[Continuation tracking]
\label{thm:continuation-tracking}
Assume that $F_\sigma\to F_0$ locally uniformly as $\sigma\downarrow 0$. Suppose the schedule $\sigma_0>\sigma_1>\cdots\downarrow 0$ is slow enough that consecutive local basins $\mathcal{B}_{\sigma_k}$ and $\mathcal{B}_{\sigma_{k+1}}$ overlap, the iterate produced at level $\sigma_k$ enters $\mathcal{B}_{\sigma_{k+1}}$, and the fixed-$\sigma_k$ contraction in Proposition~\ref{prop:fixedsigma-descent} holds on each basin. Then the continuation iterates track a path of local stationary regions for $F_{\sigma_k}$ and converge, up to the stated local assumptions, to a stationary point of the limiting objective $F_0$.
\end{theorem}

Proofs and sufficient basin-tracking conditions are in Appendix~\ref{app:section4-proofs}. Thus, locally, the two directions optimize a smoothed measurement-plus-prior objective while the annealing schedule transports this interpretation across noise levels.

\subsection{Correspondence to Bayes-rule solvers}

Bayes-rule-based posterior solvers use the conditional score decomposition
\begin{equation}
    \nabla_{x_t}\log p_{\sigma_t}(x_t\mid y)
    =
    \nabla_{x_t}\log p_{\sigma_t}(x_t)
    +
    \nabla_{x_t}\log p_{\sigma_t}(y\mid x_t).
    \label{eq:section4-bayes-score}
\end{equation}
Each term can be read as an approximate oracle for one subproblem. The prior term is
\begin{equation}
    \nabla_{x_t}\log p_{\sigma_t}(x_t)
    =
    s_{\sigma_t}(x_t)
    =
    -\nabla \phi_{\sigma_t}(x_t).
\end{equation}
The first-order prior oracle is therefore a score step:
\begin{equation}
    \mathcal{P}_{\eta_p,\sigma_t}(x)
    \;\approx\;
    x-\eta_p\nabla g_{\sigma_t}(x)
    =
    x+\eta_p\lambda s_{\sigma_t}(x).
    \label{eq:prior-oracle-score}
\end{equation}
The likelihood term gives the measurement oracle. For Gaussian measurements, a direct gradient step on $f(x;y,A)=\frac{1}{2\sigma_y^2}\norm{A(x)-y}^2$ recovers Score-ALD-style correction, while DPS uses the denoised surrogate $\ell_y(\hat{x}_0(x_t))$ \citep{chung2023dps}. In this view, different solvers instantiate the same prior-versus-measurement split with different approximate oracles. This separation also clarifies where hallucination enters: the measurement oracle anchors the reconstruction to $(y,A)$, whereas the prior oracle fills structure that may be weakly constrained or absent in the measurement. Section~5 therefore keeps the measurement oracle fixed and robustifies the prior oracle.

\section{RPU: Robust Prior Update}
\label{sec:rpu}

\begin{algorithm}[t]
\caption{DPS with Robust Prior Update (RPU)}
\label{alg:rpu}
\begin{algorithmic}[1]
\REQUIRE measurement $y$, operator $A$, diffusion model $\epsilon_\theta$, steps $N$, DPS step sizes $\{\zeta_i\}_{i=1}^N$, reverse variances $\{\tilde{\sigma}_i\}_{i=1}^N$, RPU scale $\gamma$, inner steps $K$
\STATE $x_N \sim \mathcal{N}(0,I)$
\FOR{$i=N,\ldots,1$}
    \STATE $\gamma_i \gets \gamma \tilde{\sigma}_i$, \quad $x_i^{(0)} \gets x_i$
    \FOR{$j=0,\ldots,K-1$}
        \STATE $d_j \gets m_\theta(x_i^{(j)},i)-x_i^{(j)}$
        \STATE $x_i^{(j+1)} \gets x_i^{(j)} - \dfrac{\gamma_i}{K}\dfrac{d_j}{\|d_j\|+\varepsilon}$
    \ENDFOR
    \STATE $x_i^{\mathrm{adv}} \gets x_i^{(K)}$
    \STATE $\epsilon_i^{\mathrm{adv}} \gets \epsilon_\theta(x_i^{\mathrm{adv}},i)$
    \STATE $\hat{x}_{0,i}^{\mathrm{RPU}} \gets \mathrm{Tweedie}(x_i,\epsilon_i^{\mathrm{adv}},i)$
    \STATE $\tilde{x}_{i-1}^{\mathrm{prior}} \gets x_i+\big(p_\theta(x_i^{\mathrm{adv}},i)-x_i^{\mathrm{adv}}\big)$
    \STATE $x_{i-1}\gets \tilde{x}_{i-1}^{\mathrm{prior}}
    -\zeta_i\nabla_{x_i}\|y-A(\hat{x}_{0,i}^{\mathrm{RPU}})\|_2^2$
\ENDFOR
\RETURN $\hat{x}_{0,0}^{\mathrm{RPU}}$
\end{algorithmic}
\end{algorithm}

Section~\ref{sec:alternating} shows that Bayes-rule-based diffusion inverse solvers alternate between a prior update and a measurement-conditioning step. RPU intervenes only in the prior update. The measurement operator, data-fidelity loss, and conditioning rule are kept unchanged. This makes RPU, as shown in Figure~\ref{fig:intro-roadmap}B, applicable beyond DPS in principle.

Let $x_t$ be the current reverse-diffusion iterate. Before measurement conditioning, the unconditional diffusion model proposes a native reverse transition. We denote its posterior mean and sampled reverse step by
\begin{equation}
    m_\theta(x_t,t)
    =
    \mathbb{E}_\theta[x_{t-1}\mid x_t],
    \qquad
    p_\theta(x_t,t)
    \sim
    p_\theta(x_{t-1}\mid x_t).
    \label{eq:native-prior-transition}
\end{equation}
The corresponding local prior displacement is
\begin{equation}
    d_\theta(x_t,t)
    =
    m_\theta(x_t,t)-x_t .
    \label{eq:native-prior-displacement}
\end{equation}
Vanilla DPS directly uses this native prior proposal before applying the measurement correction. RPU instead probes whether this prior motion is locally stable before accepting the update.

Starting from $x_t^{(0)}=x_t$, RPU moves a small distance against the normalized prior displacement:
\begin{equation}
    x_t^{(j+1)}
    =
    x_t^{(j)}
    -
    \frac{\gamma_t}{K}
    \frac{d_\theta(x_t^{(j)},t)}
    {\norm{d_\theta(x_t^{(j)},t)}+\epsilon},
    \qquad
    j=0,\ldots,K-1.
    \label{eq:rpu-probe}
\end{equation}
Here $K$ is the number of probe steps and $\gamma_t$ is the perturbation radius. In our implementation,
\begin{equation}
    \gamma_t
    =
    \gamma\sqrt{\mathrm{Var}(x_{t-1}\mid x_t)},
    \label{eq:probe-radius}
\end{equation}
so the probe scale follows the native reverse-step variance. The resulting point $x_t^{\rm adv}=x_t^{(K)}$ is used only to evaluate the local behavior of the diffusion prior; it does not replace the current iterate.

RPU then evaluates the diffusion prior at $x_t^{\rm adv}$ but transfers only the resulting displacement back to the original anchor $x_t$:
\begin{equation}
    \tilde{x}_{t-1}^{\rm prior}
    =
    x_t
    +
    \left(
    p_\theta(x_t^{\rm adv},t)-x_t^{\rm adv}
    \right).
    \label{eq:rpu-reanchor}
\end{equation}
This re-anchoring is the main difference from a direct adversarial perturbation of the sample. The perturbed point is used to obtain a more conservative prior displacement, while the reverse trajectory remains anchored at $x_t$ before measurement conditioning. The denoised estimate used in the measurement residual is computed consistently from the probed model output:
\begin{equation}
    \hat{x}_{0,t}^{\rm RPU}
    =
    \mathrm{Tweedie}\!\left(x_t,\epsilon_\theta(x_t^{\rm adv},t),t\right).
    \label{eq:rpu-robust-x0hat}
\end{equation}
The final conditioning step is therefore
\begin{equation}
    x_{t-1}
    =
    \tilde{x}_{t-1}^{\rm prior}
    -
    \zeta_t
    \nabla_{x_t}
    \left\|
    y-A\!\left(\hat{x}_{0,t}^{\rm RPU}\right)
    \right\|_2^2 .
    \label{eq:rpu-dps-conditioning}
\end{equation}
Thus, RPU changes the prior sample and the denoised estimate supplied to DPS, but leaves the DPS measurement correction itself unchanged.

\section{Experiments}

\subsection{Setup}
We evaluate RPU on two datasets with different roles. FFHQ is the primary benchmark for the main empirical claims. For FFHQ, we use 1,000 images at $256\times256$ resolution for each degradation and evaluate three inverse problems: box inpainting with a square mask length in $[128,129]$, Gaussian deblurring with kernel size 61 and intensity 15.0, and motion deblurring with kernel size 61 and intensity 0.5. All three FFHQ settings use Gaussian measurement noise with $\sigma_y=0.05$.

ImageNet is included as a secondary transfer check rather than as evidence for a universal improvement claim. In the main text, we report ImageNet Gaussian deblurring only, using 1,000 images at $256\times256$ resolution, kernel size 61, blur intensity 2.0, and Gaussian measurement noise $\sigma_y=0.05$. Additional ImageNet rows and qualitative examples are reported in Appendix~\ref{app:imagenet-secondary}.

\subsection{Evaluation Metrics}

We report PSNR and LPIPS~\citep{zhang2018unreasonable} against the ground truth as instance-level reconstruction metrics, and FID~\citep{heusel2017gans} as a distribution-level realism metric. Since hallucination in inverse problems is an instance-level faithfulness failure, FID is not used as a standalone measure of faithfulness: a reconstruction may look distributionally plausible while changing identity or introducing unsupported structures.

We further conduct human faithfulness studies to directly compare reconstruction fidelity. In the blind pairwise protocol, readers compare DPS and RPU outputs without seeing the ground truth. In the with-GT protocol, readers compare both reconstructions against the ground-truth image. For each protocol, we decode reader choices into DPS, RPU, or tie, and report the majority label per image, the non-tie rate, the RPU share among non-tie majority cases, and a two-sided binomial test over non-tie majority cases. Full details of the reader interface, randomization, decoding procedure, and statistical test are provided in Appendix~\ref{app:human-details}.

\subsection{Results}

We discuss four settings in parallel: FFHQ box inpainting, FFHQ Gaussian deblurring, FFHQ motion deblurring, and ImageNet Gaussian deblurring. FFHQ automatic metrics are summarized in Table~\ref{tab:ffhq_inverse_results}, FFHQ human faithfulness judgments are summarized in Table~\ref{tab:ffhq_reader_study}, and the ImageNet Gaussian result is summarized separately in Table~\ref{tab:imagenet_gaussian_transfer}.

\begin{table*}[t]
  \centering
  \caption{Quantitative evaluation of inverse problem solving on the FFHQ 256$\times$256-1k validation set.
  RPU uses the same DPS measurement update and changes only the prior-side proposal. PSNR and LPIPS are instance-level metrics; FID is a distribution-level metric and should be read together with the human faithfulness results.
  \textbf{Bold} indicates the best result in each metric among the compared methods.}
  \label{tab:ffhq_inverse_results}
  \resizebox{\textwidth}{!}{
  \begin{tabular}{lccccccccc}
  \toprule
  \multirow{2}{*}{\textbf{FFHQ}}
  & \multicolumn{3}{c}{\textbf{Inpaint (box)}}
  & \multicolumn{3}{c}{\textbf{Deblur (Gaussian)}}
  & \multicolumn{3}{c}{\textbf{Deblur (motion)}} \\
  \cmidrule(lr){2-4}
  \cmidrule(lr){5-7}
  \cmidrule(lr){8-10}
  & \textbf{FID $\downarrow$} & \textbf{LPIPS $\downarrow$} & \textbf{PSNR $\uparrow$}
  & \textbf{FID $\downarrow$} & \textbf{LPIPS $\downarrow$} & \textbf{PSNR $\uparrow$}
  & \textbf{FID $\downarrow$} & \textbf{LPIPS $\downarrow$} & \textbf{PSNR $\uparrow$} \\
  \midrule
  DPS
  & \textbf{45.44} & 0.400 & 19.72
  & \textbf{33.54} & 0.462 & 17.71
  & \textbf{35.90} & 0.440 & 18.56 \\

  RPU (ours)
  & 50.35 & \textbf{0.394} & \textbf{20.40}
  & 36.61 & \textbf{0.455} & \textbf{17.84}
  & 39.05 & \textbf{0.436} & \textbf{18.69} \\
  \bottomrule
  \end{tabular}
  }
\end{table*}

\begin{table*}[h]
  \centering
  \caption{Human faithfulness evaluation on FFHQ 256$\times$256-1k. ``RPU'' and ``DPS'' count images for which a strict reader majority prefers RPU or DPS, respectively, while ``Tie'' counts images without a strict method preference. ``Non-tie'' is the fraction of images with a strict preference for either method, and ``RPU Share'' is the fraction of non-tie cases assigned to RPU. The $p$-value is from a two-sided binomial test of RPU-vs-DPS preference balance over non-tie cases.}
  \label{tab:ffhq_reader_study}
  \setlength{\tabcolsep}{5pt}
  \renewcommand{\arraystretch}{1.08}
  \resizebox{\textwidth}{!}{
  \begin{tabular}{llrrrrrr}
  \toprule
  \textbf{Task} & \textbf{Protocol}
  & \textbf{RPU} & \textbf{DPS} & \textbf{Tie}
  & \textbf{Non-tie}
  & \textbf{RPU Share}
  & \textbf{$p$-value} \\
  \midrule
  Inpaint (box)
  & Blind pair
  & \textbf{68} & 6 & 926
  & 7.4\% & \textbf{91.9\%}
  & $2.14{\times}10^{-14}$ \\

  & With GT
  & \textbf{102} & 10 & 888
  & 11.2\% & \textbf{91.1\%}
  & $7.12{\times}10^{-16}$ \\

  \midrule
  Deblur (Gaussian)
  & Blind pair
  & 7 & 6 & 987
  & 1.3\% & 53.8\%
  & $1.00$ \\

  & With GT
  & \textbf{25} & 4 & 971
  & 2.9\% & \textbf{86.2\%}
  & $1.04{\times}10^{-4}$ \\

  \midrule
  Deblur (motion)
  & Blind pair
  & 2 & 4 & 994
  & 0.6\% & 33.3\%
  & $0.688$ \\

  & With GT
  & \textbf{23} & 5 & 972
  & 2.8\% & \textbf{82.1\%}
  & $9.12{\times}10^{-4}$ \\
  \bottomrule
  \end{tabular}
  }
\end{table*}

\paragraph{FFHQ box inpainting.}
In box inpainting, the mask leaves a large weakly constrained region, giving the diffusion prior substantial room to fill identity-relevant content. In Table~\ref{tab:ffhq_inverse_results}, RPU improves PSNR from 19.72 to 20.40 and LPIPS from 0.400 to 0.394, while FID changes from 45.44 to 50.35. In the human study, RPU is preferred in 68 versus 6 blind non-tie majority cases and 102 versus 10 with-GT non-tie majority cases. These results match the intended role of RPU: improving instance faithfulness when the prior step shapes underconstrained content.

\paragraph{FFHQ Gaussian deblurring.}
In Gaussian deblurring, RPU improves PSNR from 17.71 to 17.84 and LPIPS from 0.462 to 0.455, while FID changes from 33.54 to 36.61. Blind pairwise judgments are mostly ties, with 7 RPU majorities and 6 DPS majorities. With ground truth, readers prefer RPU in 25 versus 4 non-tie majority cases, corresponding to a 2.9\% non-tie rate and an 86.2\% RPU share among non-tie cases. The ground-truth-assisted protocol therefore captures reconstruction differences that are usually not apparent in blind comparison.

\paragraph{FFHQ motion deblurring.}
In motion deblurring, RPU improves PSNR from 18.56 to 18.69 and LPIPS from 0.440 to 0.436, while FID changes from 35.90 to 39.05. The blind pairwise study has 2 RPU majorities, 4 DPS majorities, and 994 ties. With ground truth, readers prefer RPU in 23 versus 5 non-tie majority cases, corresponding to a 2.8\% non-tie rate and an 82.1\% RPU share among non-tie cases. As in Gaussian deblurring, the with-GT protocol makes small instance-level differences easier to judge. Additional FFHQ examples are shown in Appendix~\ref{app:ffhq-morefigures}.

\begin{figure}[t]
  \centering
  \includegraphics[width=0.7\textwidth]{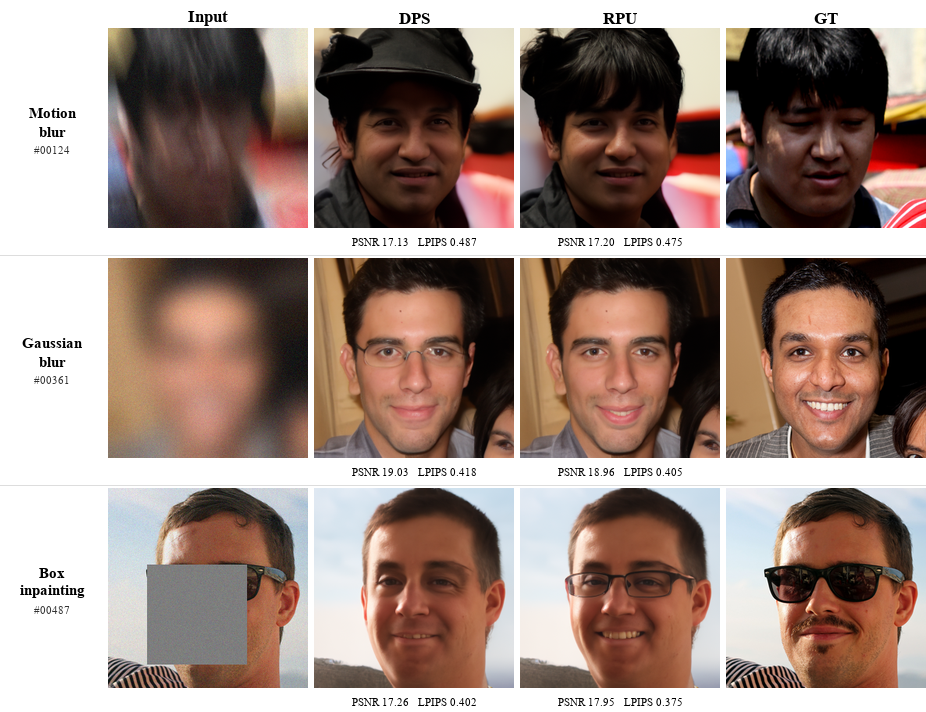}
  \caption{Qualitative examples motivating human faithfulness evaluation on FFHQ 256$\times$256-1k.
Although DPS and RPU can obtain similar PSNR and LPIPS values, human readers identify visible differences in reconstruction faithfulness.
In these examples, RPU reduces measurement-inconsistent facial details and hallucinated structures relative to DPS while preserving comparable visual realism. Rows show motion deblurring, Gaussian deblurring, and box inpainting; columns show the corrupted input, DPS, RPU, and ground truth.}
 \label{fig:ffhq_human_eval_examples}
\end{figure}

\paragraph{ImageNet Gaussian transfer check.}
On ImageNet Gaussian deblurring, RPU uses the same modification and improves PSNR from 20.406 to 20.509. LPIPS changes from 0.5342 to 0.5491 and FID changes from 83.985 to 87.296. In the three-reader with-GT study, RPU wins 8 non-tie majority cases and DPS wins 3, with 989 ties and $p=0.2266$. We view this result as a natural transfer check on a broader image distribution. Additional ImageNet quantitative rows and qualitative examples are shown in Appendix~\ref{app:imagenet-secondary}, including representative Gaussian deblurring cases in Figure~\ref{fig:imagenet-gaussian-examples}.

\begin{table*}[t]
  \centering
  \caption{ImageNet Gaussian deblurring result. The human result uses the updated three-reader with-GT study. ``R/D/T'' reports RPU, DPS, and tie majority counts.}
  \label{tab:imagenet_gaussian_transfer}
  \setlength{\tabcolsep}{4pt}
  \renewcommand{\arraystretch}{1.08}
  \resizebox{\textwidth}{!}{
  \begin{tabular}{lrrrrrrcrrrr}
  \toprule
  \textbf{Task}
  & \textbf{DPS PSNR} & \textbf{RPU PSNR}
  & \textbf{DPS LPIPS} & \textbf{RPU LPIPS}
  & \textbf{DPS FID} & \textbf{RPU FID}
  & \textbf{Readers}
  & \textbf{GT R/D/T}
  & \textbf{Non-tie}
  & \textbf{RPU Share}
  & \textbf{$p$-value} \\
  \midrule
  ImageNet Gaussian
  & 20.406 & 20.509
  & 0.5342 & 0.5491
  & 83.985 & 87.296
  & 3
  & 8/3/989
  & 1.1\%
  & 72.7\%
  & 0.2266 \\
  \bottomrule
  \end{tabular}
  }
\end{table*}

\begin{figure}[t]
  \centering
  \includegraphics[width=0.7\textwidth]{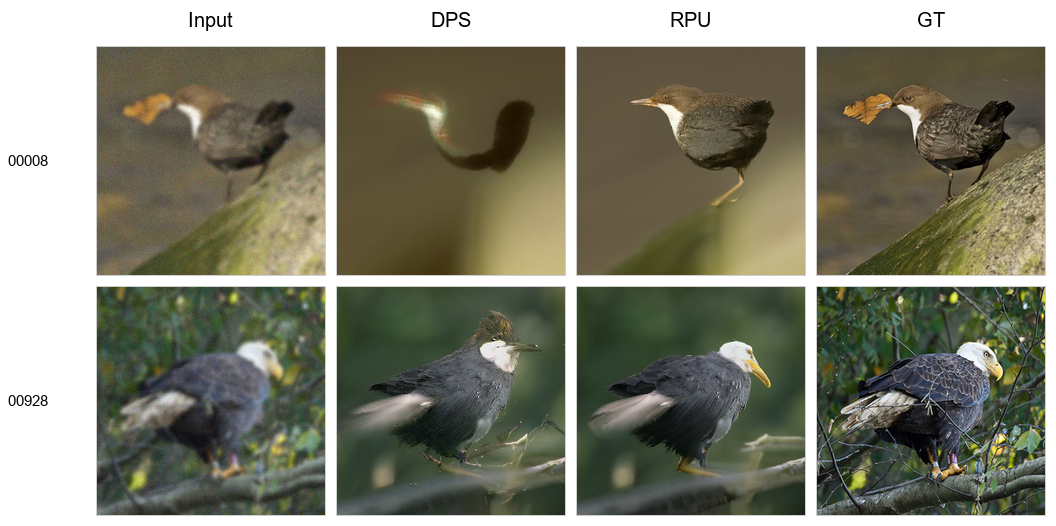}
  \caption{Representative ImageNet Gaussian deblurring examples from the ground-truth-assisted reader study. Each row shows the corrupted input, DPS reconstruction, RPU reconstruction, and ground truth.}
  \label{fig:imagenet-gaussian-examples}
\end{figure}

The qualitative evidence also explains why FID alone can be misleading: sharp, face-like DPS outputs may still change identity-relevant details, while RPU can look slightly smoother but better preserve instance-level content. The main empirical claim therefore rests on automatic metrics, reader-study preferences, and visual inspection rather than on any single score.

\section{Discussion, Limitations and Conclusion}

\paragraph{Discussion.}
Hallucination in diffusion inverse problems is not only an output-level failure; it can be shaped by the split between the diffusion-prior proposal and the measurement-conditioning step. The alternating-optimization view makes this split explicit and identifies the prior proposal as the intervention point. RPU probes and re-anchors this prior-side update while leaving the DPS measurement condition unchanged, so its effect is isolated to how generative content is injected before measurement correction.

The experiments support this mechanism. On FFHQ, RPU improves PSNR and LPIPS across box inpainting, Gaussian deblurring, and motion deblurring, while the reader studies show instance-level faithfulness gains that automatic metrics alone do not capture. The ImageNet Gaussian experiment further indicates that the same update can transfer to a broader object distribution. These results support RPU as a targeted faithfulness intervention, not as a replacement for reconstruction metrics or human inspection.

\paragraph{Limitations.}
RPU is not a standalone hallucination detector; it modifies the reverse update but does not label individual reconstructions as hallucinated. The empirical evidence is also task dependent: box inpainting is strongly underconstrained, whereas blur tasks and ImageNet yield many reader-study ties. Broader claims require more operators, datasets, reader populations, and integrations beyond the DPS instantiation studied here.

\paragraph{Conclusion.}
We defined inverse-problem hallucination as measurement-conditioned unfaithfulness, reinterpreted Bayes-rule-based solvers as alternating prior and measurement steps, and proposed RPU to robustify the prior update while keeping measurement conditioning fixed. The resulting DPS instantiation improves FFHQ metrics and produces human-preferred reconstructions, suggesting robust prior updates as a practical path toward more faithful diffusion inverse solvers.

\section*{Acknowledgment}

This work was supported by the National Institutes of Health (NIH) under award number R01HL159183. 

\bibliographystyle{unsrt}
\bibliography{references}

\newpage
\appendix

\section{Proofs and Details for Section~\ref{sec:alternating}}
\label{app:section4-proofs}

\subsection{Proof of Proposition~\ref{prop:fixedsigma-descent}}

\begin{proof}
Fix $\sigma$ and write $g=g_\sigma$, $L_g=L_{g,\sigma}$, and $F=f+g$. The split update is
\begin{equation}
    z=x-\alpha\nabla f(x),
    \qquad
    x^+=z-\beta\nabla g(z).
\end{equation}
By the descent lemma for $f$,
\begin{equation}
    f(z)-f(x)
    \le
    -\alpha\left(1-\frac{\alpha L_f}{2}\right)
    \norm{\nabla f(x)}^2 .
    \label{eq:app-f-step}
\end{equation}
By the descent lemma for $g$ across the same displacement,
\begin{equation}
    g(z)-g(x)
    \le
    -\alpha\ip{\nabla g(x)}{\nabla f(x)}
    +
    \frac{\alpha^2 L_g}{2}
    \norm{\nabla f(x)}^2 .
\end{equation}
Using Assumption~\ref{assump:fixedsigma},
\begin{equation}
    g(z)-g(x)
    \le
    \left(
    \alpha\kappa_f
    +
    \frac{\alpha^2 L_g}{2}
    \right)
    \norm{\nabla f(x)}^2 .
    \label{eq:app-g-drift}
\end{equation}
Similarly, across the prior step,
\begin{equation}
    f(x^+)-f(z)
    \le
    -\beta\ip{\nabla f(z)}{\nabla g(z)}
    +
    \frac{\beta^2 L_f}{2}
    \norm{\nabla g(z)}^2
    \le
    \left(
    \beta\kappa_g
    +
    \frac{\beta^2 L_f}{2}
    \right)
    \norm{\nabla g(z)}^2 ,
    \label{eq:app-f-drift}
\end{equation}
and
\begin{equation}
    g(x^+)-g(z)
    \le
    -\beta\left(1-\frac{\beta L_g}{2}\right)
    \norm{\nabla g(z)}^2 .
    \label{eq:app-g-step}
\end{equation}
Adding \eqref{eq:app-f-step}--\eqref{eq:app-g-step} gives
\begin{equation}
    F(x^+) \le F(x)
    -
    a_\sigma\norm{\nabla f(x)}^2
    -
    b_\sigma\norm{\nabla g(z)}^2 ,
\end{equation}
with $a_\sigma$ and $b_\sigma$ as defined in Proposition~\ref{prop:fixedsigma-descent}. This proves \eqref{eq:local-split-descent}.

For the contraction claim, note that
\begin{align}
    \norm{\nabla F(x)}
    &=
    \norm{\nabla f(x)+\nabla g(x)} \\
    &\le
    \norm{\nabla f(x)}
    +
    \norm{\nabla g(z)}
    +
    \norm{\nabla g(x)-\nabla g(z)} \\
    &\le
    (1+\alpha L_g)\norm{\nabla f(x)}
    +
    \norm{\nabla g(z)} .
\end{align}
Therefore
\begin{equation}
    \norm{\nabla F(x)}^2
    \le
    2(1+\alpha L_g)^2\norm{\nabla f(x)}^2
    +
    2\norm{\nabla g(z)}^2 .
\end{equation}
With
\begin{equation}
    c_\sigma =
    \min\left\{
    \frac{a_\sigma}{2(1+\alpha L_g)^2},
    \frac{b_\sigma}{2}
    \right\},
\end{equation}
the descent inequality implies
\begin{equation}
    F(x^+) \le F(x)-c_\sigma\norm{\nabla F(x)}^2 .
\end{equation}
If $F$ satisfies the local PL inequality
\begin{equation}
    \frac{1}{2}\norm{\nabla F(x)}^2
    \ge
    \mu_\sigma\left(F(x)-F_\sigma^\star\right),
\end{equation}
then
\begin{equation}
    F(x^+)-F_\sigma^\star
    \le
    (1-2\mu_\sigma c_\sigma)
    \left(F(x)-F_\sigma^\star\right),
\end{equation}
which proves \eqref{eq:fixed-sigma-contraction}.
\end{proof}

\subsection{Continuation tracking conditions}
\label{app:continuation-tracking}

Theorem~\ref{thm:continuation-tracking} is a conditional basin-tracking theorem rather than a global convergence theorem. One sufficient set of conditions is the following. Let $\mathcal{B}_{\sigma_k}$ be local basins for the smoothed objectives $F_{\sigma_k}$. Assume:
\begin{enumerate}[leftmargin=1.5em]
    \item the fixed-$\sigma_k$ contraction in Proposition~\ref{prop:fixedsigma-descent} holds uniformly inside $\mathcal{B}_{\sigma_k}$;
    \item consecutive basins overlap, and the output of the $\sigma_k$ stage lies inside $\mathcal{B}_{\sigma_{k+1}}$;
    \item the objective drift is controlled on the relevant basin, for example
    \begin{equation}
        \sup_{x\in \mathcal{B}_{\sigma_k}\cap\mathcal{B}_{\sigma_{k+1}}}
        \left|
        F_{\sigma_{k+1}}(x)-F_{\sigma_k}(x)
        \right|
        \le
        \Delta_k,
        \qquad
        \sum_k \Delta_k < \infty;
    \end{equation}
    \item $F_{\sigma_k}\to F_0$ locally uniformly as $\sigma_k\downarrow 0$.
\end{enumerate}
Under these assumptions, the iterate can be passed from one local basin to the next while preserving the fixed-$\sigma$ descent behavior up to the summable drift terms. Thus the continuation procedure tracks a path of local stationary regions for the smoothed objectives. In the limit $\sigma_k\downarrow 0$, any accumulation point that remains in the tracked basin is a stationary point of the limiting objective $F_0$, provided the limiting objective is differentiable at that point. If $F_0$ is nonsmooth, the same argument should be read as convergence to the corresponding local limiting-stationarity condition.

\section{FFHQ results} \label{app:ffhq-morefigures}

\begin{figure}[t]
  \centering
  \includegraphics[width=\textwidth]{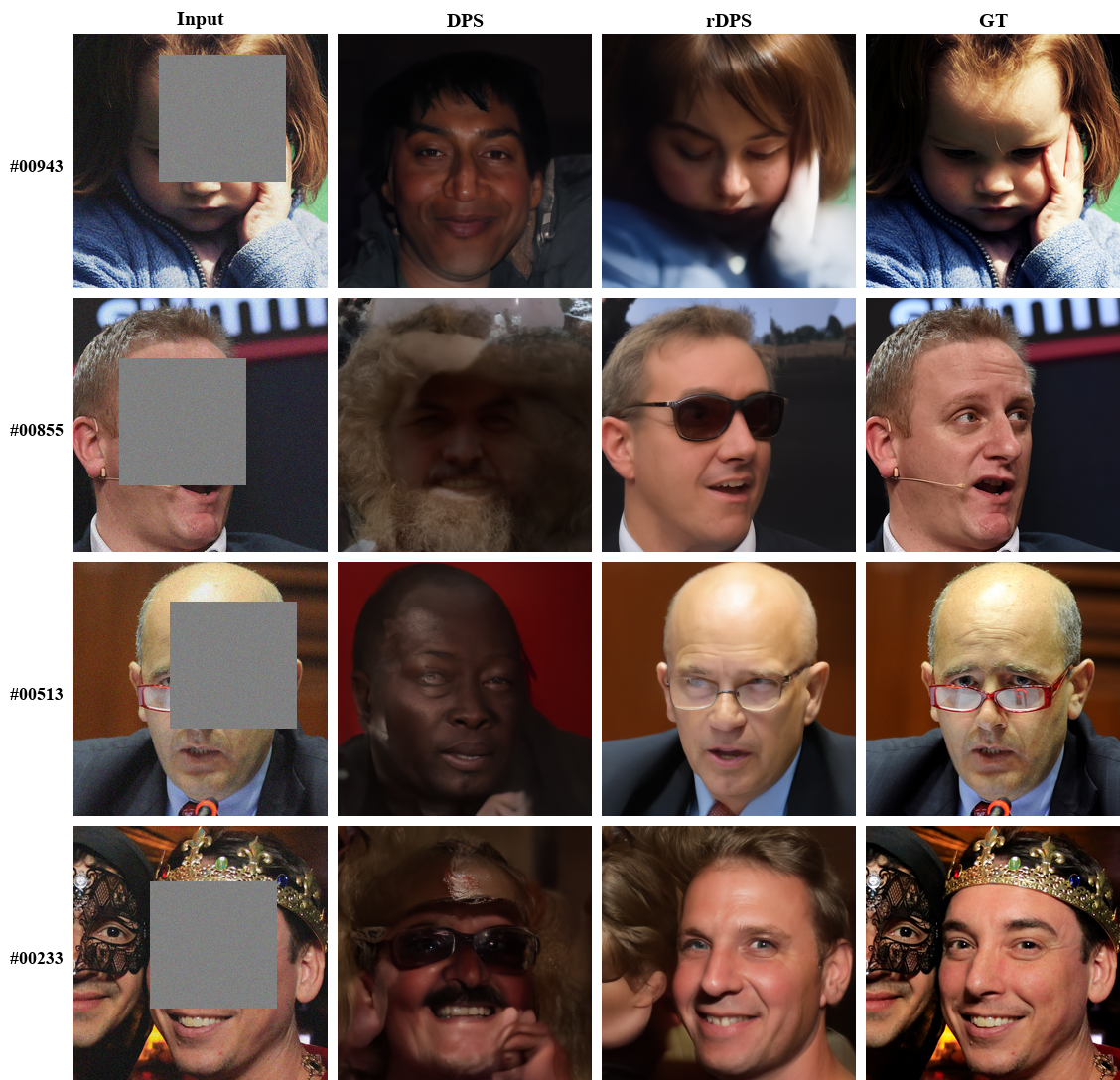}
  \caption{Qualitative FFHQ box-inpainting examples. Each row shows the corrupted input, DPS reconstruction, RPU reconstruction, and ground truth. Rows are selected from cases with unanimous RPU preference in the ground-truth-assisted reader study and high DPS--RPU output divergence. The figure is intended as visual evidence for the mechanism behind the inpainting result, not as a substitute for the aggregate reader-study statistics in Table~\ref{tab:ffhq_reader_study}.}
  \label{fig:qual-inpaint}
\end{figure}

\begin{figure}[t]
  \centering
  \includegraphics[width=\textwidth]{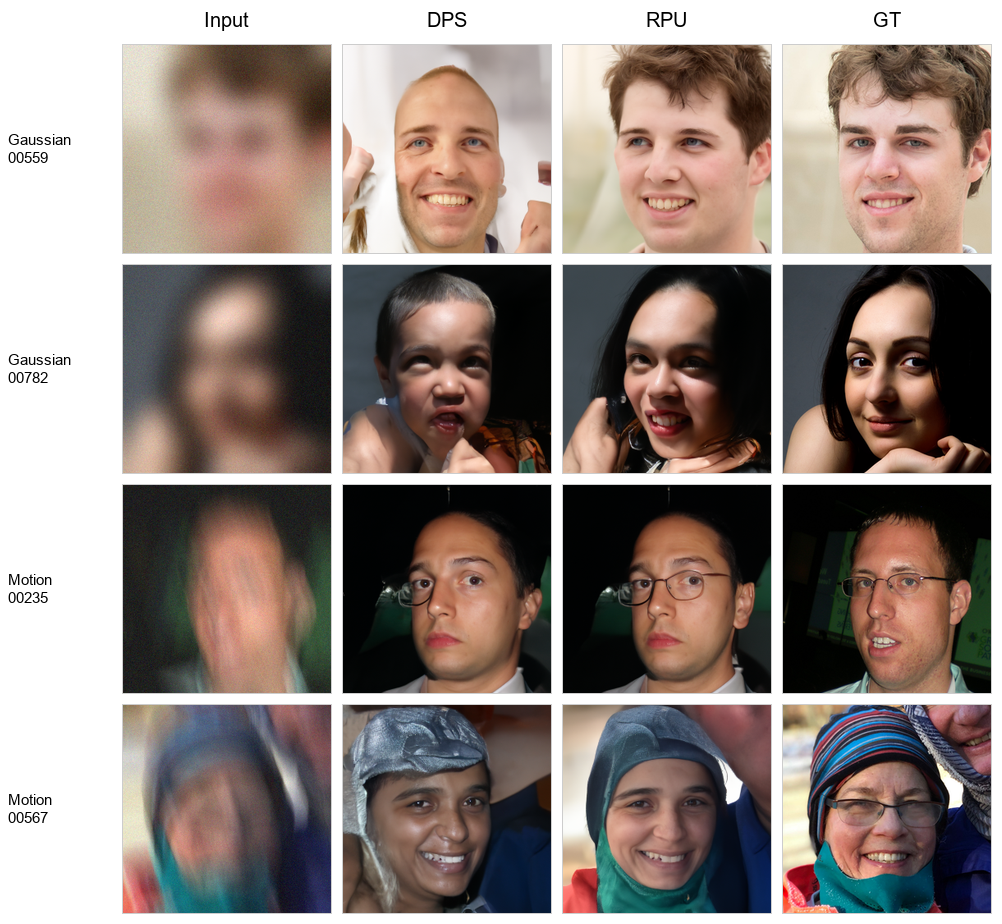}
  \caption{Additional FFHQ blur examples from ground-truth-assisted reader-study cases. Each row shows the corrupted input, DPS reconstruction, RPU reconstruction, and ground truth for Gaussian or motion deblurring. These examples complement the aggregate FFHQ results and illustrate cases where RPU changes the prior-side reconstruction while using the same measurement condition as DPS.}
  \label{fig:ffhq-blur-appendix-examples}
\end{figure}

\section{ImageNet results} \label{app:imagenet-secondary}

ImageNet is used as a secondary transfer check. Table~\ref{tab:imagenet-secondary-results} reports the same DPS--RPU comparison for Gaussian deblurring, motion deblurring, and box inpainting. The main text reports Gaussian deblurring because it has the updated three-reader with-GT study; the additional rows are included here to document the broader ImageNet behavior without changing the main empirical focus.

\begin{table}[h]
\centering
\caption{Secondary ImageNet 256$\times$256-1k results. ``GT R/D/T'' reports RPU, DPS, and tie majority counts in the ground-truth-assisted reader study when available.}
\label{tab:imagenet-secondary-results}
\small
\setlength{\tabcolsep}{3.5pt}
\resizebox{\textwidth}{!}{
\begin{tabular}{lrrrrrrcrrrr}
\toprule
\textbf{Task}
& \textbf{DPS PSNR} & \textbf{RPU PSNR}
& \textbf{DPS LPIPS} & \textbf{RPU LPIPS}
& \textbf{DPS FID} & \textbf{RPU FID}
& \textbf{Readers}
& \textbf{GT R/D/T}
& \textbf{Non-tie}
& \textbf{RPU Share}
& \textbf{$p$-value} \\
\midrule
Gaussian blur
& 20.406 & 20.509
& 0.5342 & 0.5491
& 83.985 & 87.296
& 3
& 8/3/989
& 1.1\%
& 72.7\%
& 0.2266 \\
Motion blur
& 19.804 & 19.954
& 0.5381 & 0.5523
& 81.345 & 85.744
& 2
& 5/6/989
& 1.1\%
& 45.5\%
& 1.0000 \\
Box inpainting
& 21.001 & 21.059
& 0.5357 & 0.5485
& 94.523 & 98.180
& 2
& 20/16/964
& 3.6\%
& 55.6\%
& 0.6177 \\
\bottomrule
\end{tabular}
}
\end{table}

\section{Reader Study Details}
\label{app:human-details}

\paragraph{Reader-study protocols.}
Each FFHQ reader study contains 1,000 items and three readers. In the blind pairwise protocol, readers compare the DPS and RPU outputs without seeing ground truth. In the with-GT protocol, readers compare both reconstructions against the ground-truth image. Method identities are decoded using the saved \texttt{key.csv} files rather than image order. The key-alignment audit reports 1,000 unique key rows, 1,000 rows per reader, no missing or extra reader rows, and aligned keys for all six FFHQ studies.

\paragraph{Blind pairwise rubric.}
For each item, the reader opens the paired reconstruction image and enters one score in \texttt{responses\_template.csv}. The method key is hidden during annotation. Table~\ref{tab:blind-rubric} gives the exact rubric. Scores 1 and 2 are decoded through \texttt{key.csv}: score 1 selects the left method and score 2 selects the right method. Scores 0 and 3 are both treated as tie at the decoded preference level, but they are preserved separately in the raw rating-category counts.

\begin{table}[h]
\centering
\caption{Blind pairwise reader-study rubric. Readers see only the two reconstructions and do not see method identities or ground truth.}
\label{tab:blind-rubric}
\small
\begin{tabular}{cll}
\toprule
Score & Reader-facing meaning & Decoded preference use \\
\midrule
0 & Visually same & Tie \\
1 & Left better & Method on left in \texttt{key.csv} \\
2 & Right better & Method on right in \texttt{key.csv} \\
3 & Different, but cannot say which is better & Tie \\
\bottomrule
\end{tabular}
\end{table}

\paragraph{Ground-truth-assisted rubric.}
For the with-GT protocol, each panel is ordered as input, reconstruction A, reconstruction B, and ground truth. The input image gives the degraded-measurement context, while the ground truth is used to judge instance faithfulness. Table~\ref{tab:withgt-rubric} gives the exact rubric. Scores 1 and 2 are decoded through \texttt{key.csv}: score 1 selects method A and score 2 selects method B. Scores 0, 3, and 4 are treated as tie at the decoded preference level because they do not choose one method over the other, but the raw categories are kept to distinguish ``both faithful'', ``both unfaithful and similar'', and ``both unfaithful and different'' cases.

\begin{table}[h]
\centering
\caption{Ground-truth-assisted reader-study rubric. Panels are ordered as input, A, B, and ground truth.}
\label{tab:withgt-rubric}
\small
\resizebox{\textwidth}{!}{
\begin{tabular}{cll}
\toprule
Score & Reader-facing meaning & Decoded preference use \\
\midrule
0 & A and B both look like ground truth & Tie \\
1 & A looks more like ground truth & Method A in \texttt{key.csv} \\
2 & B looks more like ground truth & Method B in \texttt{key.csv} \\
3 & Both do not look like ground truth, but A and B are similar & Tie \\
4 & Both do not look like ground truth, and A and B are different & Tie \\
\bottomrule
\end{tabular}
}
\end{table}

\paragraph{Preference decoding and statistical tests.}
For each image, the three reader scores are decoded into DPS, RPU, or tie preferences using the hidden method key. The main table reports the majority label per image. If no method receives a non-tie majority, the image is counted as a tie-majority case. Binomial tests and Wilson confidence intervals are computed only over non-tie majority cases, because tie-majority cases do not express a directional method preference. Raw score counts are reported separately in Table~\ref{tab:raw-ratings} so that tie-heavy behavior is visible rather than hidden by preference decoding.

\begin{table}[h]
\centering
\caption{Raw rating-category counts for the FFHQ reader studies. The decoded preference table in the main text collapses these categories into DPS, RPU, and tie, but the raw categories are preserved for protocol auditing.}
\label{tab:raw-ratings}
\small
\resizebox{\textwidth}{!}{
\begin{tabular}{llrrrrr}
\toprule
Degradation & Study & rating 0 & rating 1 & rating 2 & rating 3 & rating 4 \\
\midrule
Box inpainting & Blind pair & 2606 & 136 & 133 & 125 & -- \\
Box inpainting & With GT & 2518 & 161 & 168 & 104 & 49 \\
Gaussian blur & Blind pair & 2858 & 39 & 30 & 73 & -- \\
Gaussian blur & With GT & 2648 & 50 & 56 & 224 & 22 \\
Motion blur & Blind pair & 2916 & 22 & 20 & 42 & -- \\
Motion blur & With GT & 2742 & 46 & 53 & 150 & 9 \\
\bottomrule
\end{tabular}
}
\end{table}

\paragraph{Agreement statistics.}
Fleiss' kappa is highest for box inpainting with ground truth, where $\kappa=0.747$, and lower for the blur tasks, where the dominant tie category reduces the number of informative preference cases. The corresponding blind-pair kappas are 0.568 for box inpainting, 0.299 for Gaussian blur, and 0.182 for motion blur. With ground truth, the kappas are 0.747, 0.484, and 0.455, respectively.

\end{document}